\documentclass[letterpaper, 10 pt]{IEEEtran}
\IEEEoverridecommandlockouts 

\usepackage{amsmath,amssymb,pstricks,color,dsfont,amsthm}
\usepackage{graphicx}      
\usepackage{psfrag,graphicx,epsfig}
\usepackage{epstopdf}
\usepackage{xspace}
\usepackage{subfig}
\usepackage{float}
\usepackage{placeins}
\usepackage{multirow}
\usepackage{pgf,tikz}
\usepackage{nowidow}
\usepackage[vlined,linesnumbered,boxruled]{algorithm2e}
\IEEEoverridecommandlockouts
\usepackage{algpseudocode}

\usepackage{etex}
\usepackage{array}
\usepackage{cite}
\newcolumntype{x}[1]{%
	>{\centering\hspace{0pt}}p{#1}}%

\newcommand{\EQ}{\begin{eqnarray}}
	\newcommand{\EN}{\end{eqnarray}}
\newcommand{\EQQ}{\begin{eqnarray*}}
	\newcommand{\ENN}{\end{eqnarray*}}

\newcommand{\real}{\mathds{R}}

\renewcommand{\natural}{\mathds{N}}

\newtheorem{theorem}{\bf Theorem}[section]

\newtheorem{lemma}{\bf Lemma}[section]
\newtheorem{definition}{\bf Definition}[section]
\newtheorem{remark}{\bf Remark}[section]

\newtheorem{assumption}{\bf Assumption}[section]

\title {
{\Large \bf
Distributed Task Allocation in Homogeneous Swarms Using Language Measure Theory
}}

\begin{document}

\author{\begin{tabular}{c}
 Devesh K. Jha$^{\dag}$\\
 MERL, Cambridge, MA, USA 
\end{tabular}\\ 
\thanks{ $^\dag$ Email: \small \tt jha@merl.com}
}


\maketitle


\begin{abstract}

In this paper, we present algorithms for synthesizing controllers to distribute a group (possibly swarms) of homogeneous robots (agents) over heterogeneous tasks which are operated in parallel. We present algorithms as well as analysis for global and local-feedback-based controller for the swarms. Using ergodicity property of irreducible Markov chains, we design a controller for global swarm control. Furthermore, to provide some degree of autonomy to the agents, we augment this global controller by a local feedback-based controller using Language measure theory. We provide analysis of the proposed algorithms to show their correctness. Numerical experiments are shown to illustrate the performance of the proposed algorithms.

\end{abstract}

\section{Introduction}\label{sec:intro}
\textit{Swarms} of robots are increasingly becoming popular with recent advances in design of embedded processor. These are, in general, designed to perform different complex tasks in a coordinated manner. 
 Some important examples are fleets of unmanned vehicles deployed in ocean bed for the purposes of mine-hunting or surveillance or drug delivery techniques in humans using micro-scale robots at pre-specified rate and locations~\cite{WPN14,RCN14}. A common element in all these applications is a desirable global behavior which can be achieved where the individual behavior is achieved by the agents themselves and thus, is not considered during global policy design. In these large-scale systems, it is difficult to design controller for each agent (or robot) individually and then coordinate their behavior due to computational requirements. A critical challenge for these distributed systems is to guarantee stability and performance with limited global information. In this paper, we are interested in design and analysis of controllers for controlling global states of a large-scale system with the assumption that each of the individual agents are controllable. Such a controller could be designed using a multi-agent motion planning algorithm~\cite{jha2015game,wang2017two}. 

Following our previous work in~\cite{J18}, 
we model swarm of robots (which we also interchangeably call agents) as a homogeneous collection of irreducible Markov chains. The states of the Markov chain are the heterogeneous tasks that are operated in parallel. The state of the system (swarm) is defined by the distribution of agents over the tasks. A desired system state is then the desired distribution of agents over the heterogeneous tasks. The control problem is to design a policy that the agents can follow to transition between the tasks so that the distribution of agents over tasks converges to the desired distribution. Consequently, the swarm control problem is reduced to achieving a desired probability stationary distribution of the Markov chain. The design problem is to synthesize the transition kernel for the Markov chain with a pre-specified stationary distribution over its states. We present a closed-form solution to estimate a transition kernel for a Markov chain with a known stationary distribution. We compute a state transition matrix to serve as the global policy for the agents-- agents transition between states using this global policy. This policy guarantees convergence to the desired global state. The agents then calculate the expected reward of moving to a different task against staying in their current task based on the excess or deficit of agents in the current and neighboring states. This local policy of an agent is calculated as a perturbation to its global policy which is computed using Language measure theory of Probabilistic Finite State Automata (PFSA)~\cite{R05, CR07, JLWR15}. It is noted that in the following text we often call transition matrix as Markov kernel even though kernels are generally defined for continuous state-space (and our problem is discrete).  

\textbf{Contributions:} An initial conference version of the paper was presented in~\cite{J18}. This paper has the following contributions over the previous publication.
\begin{enumerate}

\item We present analysis of the algorithms for centralized as well as design of stable local controllers for swarm control. 
\item We present empirical results about performance and stability of the proposed controllers. We provide some empirical evidence about choice of parameters for such a controller.

\end{enumerate} 
For the completeness of this work, we provide a full problem description which was also described earlier in~\cite{J18}. 



   \section{Related Work}\label{sec:related_work}
Traditionally, swarm control has been studied under two broad categories: centralized control with global information, where the controller broadcasts the policies to the swarm and decentralized control where some bio-inspired collective behavior laws are used to replicate social behavior in nature~\cite{khamis2015multi}. A lot of work has been done in swarm control addressing issues of centralized and decentralized control. Some examples of centralized control could be found in~\cite{BHHK09,CR09,NK11,AB12,BKN12,AYS13,ito2020pseudo}. The common idea in the centralized control techniques is to synthesize stochastic policy for agents which they can use to switch between the tasks such that the system is driven to the desired steady distribution of agents over the tasks. This approach has the shortcomings that we need global state information. Moreover, controller synthesis becomes restrictive for large systems with large number of states or tasks. Furthermore, it is difficult to make such a system react to local un-modeled dynamics and other disturbances. 

These restrictions have driven a lot of research effort in distributed control approaches to circumvent the shortcomings of the centralized approaches. Some distributed control approaches for swarm control could be found in~\cite{JGAJ12,FJHJ15,HHBK08,CDLN14, jang2018local, jang2018anonymous, bandyopadhyay2017probabilistic, khamis2015multi}. Most of the distributed strategies for swarm control involve using bio-inspired logic to design reactive controllers for desired global behavior. This work is motivated by the idea that we can augment the centralized controller by a state-dependent feedback strategy so that the agents can use the centralized control strategy with limited perturbation from the same; the perturbation in the strategy depends on the current state of an agent and the desired steady state of the system. The distributed state-information-based feedback control leads to a time-varying stochastic linear system. Such time varying stochastic analysis has been been extensively presented in consensus literature~\cite{TN14,BXMS13}. 

While our current problem is mainly motivated and closest to earlier work in~\cite{CR09}, the central control problem is similar to previous work in~\cite{AYS13}. The decentralized control problem is similar to the solution earlier presented in~\cite{FJHJ15}. We present a simplified solution to similar problems studied in these previous papers. The proposed work presents a novel solution using language measure theory and shows some analysis for the stability of the resulting controller.

An initial version of this work was earlier presented at a conference~\cite{J18}. However, in the current paper we present the proof for the correctness of the centralized algorithm which was missing in~\cite{J18}. Furthermore, we show that the distributed framework can get stuck in limit cyclic (or oscillatory) behavior if the feedback gains are not chosen appropriately. We provide a proof sketch for the stability of the distributed system and show some empirical results to demonstrate its correctness. We present several new empirical results regarding stability of the distributed system that depends on controller parameters. We believe that the current work provides more insight to the whole framework than the preliminary version. 


\section{Problem Formulation}\label{sec:ProbForm}
Consider a set of $N$ robots to be allocated among $M$ heterogeneous tasks which are operated in parallel. The number of robots performing task $i\in \{1,\dots,M\}$ at a time epoch $k$ is denoted by $n_i^{[k]}$ . The desired number of robots for task $i$ is denoted by $n^d_i$. We assume that $n_i^d>0$ for all $i \in \{1,\dots,M\}$. Then, to make the system scalable in the number of agents, we define the population fraction at any task $i$ as $p^{[k]}_i=n_i^{[k]}/N$. The state of the swarm is then defined as $\mathbf{p}^{[k]}=[p_1^{[k]},\dots,p_M^{[k]}]$. The desired state of the swarm is given by the fraction of agents at the individual tasks which is denoted by the vector $\mathbf{p}^d=[p_1^d,\dots,p_M^d]$. It is noted that since $n_i^d > 0$ $\forall i \in \{1,2,,\dots,M\}$, thus we have that $p_i^d>0\, \forall \, i \in \{1,\dots, M\}$. 
\begin{assumption}
$|Q|=M <\infty$ i.e., the number of tasks operated in parallel are finite.
\end{assumption}
Next we formalize the definition of an agent and swarm before we state the formal problem.
\begin{definition}\label{def:agent}
\textbf{(Agent in Swarm Modeling)}: An agent is a connected digraph $R=(Q,E)$, where each state $i\in Q$ represents a distinct predefined behavior (i.e., a heterogeneous task), and $E\in \{0,1\}^{|Q|\times |Q|}$ is a matrix such that $E_{ij}=1$ implies there exist a controllable transition from state $i$ to $j$ (implying the connectivity of the tasks). 
\end{definition}

The matrix $E$ in Definition~\ref{def:agent} specifies state transitions of the agent $R$'s state. An agent can only transition to the tasks it is directly connected with in a single hop (or a controllable movement). It is possible to associate probabilities with the transition of an agent between tasks based on the requirements at the tasks or the preference of the robot. The probabilities of the state transitions constitute a finite-state irreducible Markov chain, with the irreducibility property following from the connectedness of the agent graph. Thus the behavior of an agent can be  represented by an irreducible Markov chain $G=(Q, \textbf{P} ,\mathbf{p}^{[0]})$, where $\textbf{P}$ represents a stochastic matrix such that $\textbf{P}_{ij}>0$ if and only if $E_{ij}=1$; $\mathbf{p}^{[0]}$ represents its initial state (which could be a one-hot vector representing the initial task of the agent). The swarm is a collection of such agents, and thus could be denoted as $\mathds S=\{G^{\alpha}:\alpha\in \mathbb{X}\}$, such that $G^{\alpha}=G$, and $\mathbb{X}$ is an index set (finite, countable or uncountable). Consequently, the state of the swarm is defined by the distribution of robots over the tasks i.e., $\mathbf{p}^{[k]}$. The swarm dynamics is represented by the following equation:
\begin{equation}\label{eqn:swarm}
\mathbf{p}^{[k+1]}=\mathbf{p}^{[k]} \textbf{P}
\end{equation}
where $\textbf{P}_{ij}$ represents the probability with which an agent decides to switch from task $i$ to task $j$.

\subsection{Central Controller Design}\label{sec:problem1}
Consider a swarm of robots where the initial state of the swarm is $\mathbf{p}^{[0]}$ and the desired state is represented by $\mathbf{p}^d$. The goal of the central controller is to achieve the desired distribution of agents over the tasks. This can be achieved in a probabilistic setting by using a Markov kernel for transition between tasks such that the desired state is the stationary distribution over the states of the swarm. Thus, given a desired state distribution $\mathbf{p}^d$, the problem is to synthesize a Markov kernel $\textbf{P}^\star$ given any initial kernel $\textbf{P}$ such that the following conditions hold. 
\begin{enumerate}
\item $\sum\limits_{j=1}^M \textbf{P}^\star_{ij}=1$ (i.e., the controller matrix should be row stochastic).
\item $\textbf{P}^\star_{ij}>0$ if and only if $\textbf{P}_{ij}>0$ for any $i,j$ $\in$ ${1,\dots ,M}$.
\item $\textbf{P}^\star$ is an irreducible matrix.
\item $\lim\limits_{k\rightarrow \infty}\|\mathbf{p}^{[0]}(\textbf{P}^\star)^k-\mathbf{p}^d\|_{\infty}=0$.
\end{enumerate}
The initial kernel $\textbf{P}$ contains the information about the connectivity of the underlying tasks and one such kernel could be easily obtained by normalizing the adjacency matrix for the swarm. It is noted that the second condition is required to maintain the connectivity of the tasks that the agents are supposed to perform. Irreducibility of $\textbf{P}^\star$ implies that the limiting time average of the local state is independent of the initial conditions. The last condition ensures the convergence to the desired behavior. This defines the central control problem for swarm control. 
\subsection{Central Control with Distributed Autonomy}\label{sec:problem2}
The next problem is inspired by minimizing the movement of agents to achieve a desired system state as well as during steady state after the desired state is achieved ( due to energy limitations). For these considerations, we want to have a local-information-based policy for the agents which could be calculated as a perturbation to the central control policy $\textbf{P}^\star$. The local communication-based information is used to calculate perturbations to the Markov kernel calculated in the first step as a proportional feedback. Given the state-dependent stochastic policy by the central controller, the agents decide between following the central controller and staying in their current state; this is decided as a function of deficit or excess of agents in their current and neighboring tasks (states) when compared against the desired distribution. The perturbed policy could be derived as a function of $\mathbf{p}_i^k$, the desired state $\mathbf{p}^d_i$ and the neighboring states $\mathbf{p}_j^k$ where $j$ is such that $E_{ij}=1$. Then, it results in time varying stochastic policies for the agents depending on their current state and local information. We denote the perturbed local policies by $\tilde{\textbf{P}}^{[k]}$ at an instant $k$. Also, we want only $\lambda$ fraction of the agents to switch between the tasks at steady state when compared to the number of agents switching tasks under the central policy. Then, the local-information-based perturbed policy has to satisfy the following conditions.
\begin{enumerate}
\item $\sum\limits_{j=1}^M\tilde{\textbf{P}}^{[k]}_{ij}=1$ (i.e., the resulting matrix should be row stochastic at all instants).
\item $\tilde{\textbf{P}}^{[k]}_{ij}>0$ if and only if $\textbf{P}^\star_{ij}>0$ for any $i,j$ $\in$ ${1,\dots ,M}$..
\item $\sum\limits_{j=1}^{|Q|}\tilde{\textbf{P}}^{[k]}_{ij}=1$.
\item $\lim\limits_{k\rightarrow \infty}\|\mathbf{p}^{[0]}\tilde{\textbf{P}}^{[k]}\tilde{\textbf{P}}^{[k-1]}\dots\tilde{\textbf{P}}^{[0]}-\mathbf{p}^d\|_{\infty}=0$
\item $\lim\limits_{k\rightarrow\infty}\tilde{\textbf{P}}^{[k]}_{ij}=\lambda \textbf{P}_{ij}^\star$ for $i\neq j$ and $\lim\limits_{k\rightarrow\infty}\tilde{\textbf{P}}^{[k]}_{ii}=\lambda \textbf{P}_{ii}^\star+(1-\lambda) $
\end{enumerate}
It is noted that the irreducibility of $\tilde{\textbf{P}}^{[k]}$ follows from condition $(1)$ and the fact that $\textbf{P}^\star$ is an irreducible matrix. At this point, we would like to clarify that condition $5$ implies that the robots stay at the same task with an increased probability of $(1-\lambda)$ and thus, the probability to switch task at any state is reduced by fraction $\lambda$ (this defines the autonomy of individual agents).

Both problems are related to synthesis of Markov kernels such that the stationary distribution of the underlying Markov chain achieves the desired state of the swarm in the asymptotic limit. In the following sections, we present solution as well as analysis of some of the proposed algorithms.


\section{Proposed Algorithms and Analysis}\label{sec:proposedapproach}
In this section we present the proposed approach for estimation of the Markov kernels described in section~\ref{sec:ProbForm}. 
\subsection{Algorithm and Analysis for Central Controller Synthesis}\label{sec:closedsoln}
This section presents an analytical solution to solve the control problem described in section~\ref{sec:problem1}. 
Let $\mathbf{p}^d$ be the desired state of the swarm which begins in the state $\mathbf{p}^{[0]}$. Let $\textbf{P}$ be an irreducible stochastic matrix for $G$ and let $\tilde{\mathbf{p}}$ be its unique stationary probability distribution vector.  Then, a Markov kernel which achieves the desired distribution over the swarm in the asymptotic limit could be obtained using the following transformation of the matrix $\textbf{P}$.
$$\textbf{P}_{ij} \longmapsto d_i \textbf{P}_{ij}, i\neq j $$
$$\textbf{P}_{ij} \longmapsto d_i \textbf{P}_{ij}+(1-d_i), i=j $$
where, $d_i \in (0,1)$   $\forall i \in \{1,2,\dots,M\}$ where the vector $\mathbf{d}=[d_1,\dots,d_M]$ is given by the following expression. 

\begin{equation}\label{eqn:closedsoln}
[d_1,\dots,d_M]=[\hat{d}_1,\dots,\hat{d}_M]/\sum_{i=1}^M \hat{d}_i 
\end{equation}
where $[\hat{d}_1,\dots,\hat{d}_M]=[p^{[0]}_1,\dots,p^{[0]}_M]\mathbf{X}^{-1}$ where, $\mathbf{X}^{-1}={\rm diag}(p^d_1,\dots,p^d_M)^{-1}$. The vector $\mathbf{p}^d$ and $\mathbf{p}^{[0]}$ denote the final and initial distribution of the swarm, respectively. The vector $\mathbf{d}$ in equation~(\ref{eqn:closedsoln}) gives the closed form solution to the iterative Algorithm 1 in~\cite{CR09}. This helps in simplifying the controller synthesis complexity which is useful for swarms with a large number of states. It is noted that the perturbations preserve the original topology of the graph representing the connectivity of the tasks. For convenience of presentation, the transformation is presented as a psuedo-code in Algorithm~\ref{algorithm:ClosedForm}. Clearly, the complexity of the Algorithm~\ref{algorithm:ClosedForm} is $O(M)$ where $M$ is the size of the swarm.

\begin{algorithm}[t] \small
\SetKwInput{KwIn}{Input}
\SetKwInput{KwOut}{Output}
	\KwIn{The initial stochastic transition matrix $\mathbf{P}$, the desired probability $\mathbf{p}^d$.}
	\KwOut{The global broadcast policy kernel $\mathbf{P}^\star$.}
	\For{$i \in \{1,.\dots,M\}$}{
	$\hat{d_i}=\frac{p^{[0]}_i}{p^d_i}$\;}
	\For{$i \in \{1,.\dots,M\}$}{
	Normalize $d_i=\frac{\hat{d_i}}{\sum_{i=1}^M \hat{d_i}}$\;}
	\Return{$\{d_i\}_{i=1,\dots,M}$\;}
	\Return{$\mathbf{P}^\star=\tt {diag}(d)\mathbf{P}-\tt {diag}(d)+I$}
    \Comment{$\tt I$ represents the identity matrix of the same size as $\mathbf{P}$}\\
	\caption{${\tt Estimating Markov Kernel}$ }
	\label{algorithm:ClosedForm}
\end{algorithm}	

Next, we present the analysis of the above results to show its correctness. The claim that we are going to prove is stated in Theorem~\ref{thm:centralcontrol}.  We first present a lemma which is required to prove our main claim of this section.  
\begin{lemma}\label{lemma1}
We define a perturbation of an irreducible matrix $\textbf{P}$ as follows:
\begin{equation}
 \hat{\mathbf{P}} = D\mathbf{P}-D+I
{\label{eq:transformation}}
\end{equation}

where,  $D=\rm {diag}(d_i)\, i \in \{1,\dots,M\}$, where $d_i\in (0,1)$ and $I$ is an identity matrix of size $M \times M$. Then $\hat{\mathbf{P}}$ is a stochastic irreducible matrix. 
\end{lemma}
\begin{proof} Clearly, $\mathbf{P}_{ij}\geq 0$. Since $D$ is a diagonal matrix, and $\Pi$ is a stochastic matrix, we have
\begin{equation}
 \sum_j \hat{\mathbf{P}}_{ij}=a_i\sum_j \mathbf{P}_{ij}+(1-d_i)=d_i+(1-d_i)=1
{\label{eq:summation}}
\end{equation}

This shows that $\mathbf{P}$ is a stochastic matrix. We now show that $\mathbf{P}$ is also irreducible. If we can show that $\hat{\mathbf{P}}$ has a unique probability vector $\hat{{P}}$ that is element-wise positive, it would imply $\hat{\mathbf{P}}$ is irreducible. Let $\tilde{{P}}$ be an element-wise non-negative vector representing a direction in the eigen-space of $\hat{\mathbf{P}}$ corresponding to its unity eigenvalue, i.e., $\tilde{{P}}[\hat{\mathbf{P}}-I]=0.$ Such a $\tilde{{P}}$ is guaranteed to exist for any stochastic matrix. Equation~{\eqref{eq:transformation}}  yields
\begin{equation}
 \tilde{{P}}[D\mathbf{P}-D+I-I]=0
\end{equation}
\begin{equation}
 \Rightarrow \tilde{{P}}D[\mathbf{P}-I]=0
{\label{eq:stat_vect}}
\end{equation}

 Since, $\mathbf{P}$ is an irreducible stochastic matrix, it has an unique stationary probability vector, element-wise positive, ${P}$. From the uniqueness of the probability vector ${P}$, it follows that, 
\begin{equation}
 \tilde{{P}}D={P}.
\end{equation}
Define,
\begin{equation}
 \hat{{P}}=\frac{1}{\sum_i \tilde{{P}}_i} \tilde{{P}}
{\label{eq:P_vector}}
\end{equation}
where, $\tilde{{P}}={P}D^{-1} $. It is noted that $D$ is a diagonal matrix with nonzero entries and hence invertible. Then, clearly, $\|\hat{P}\| _1 =1 $. Also,  $\hat{P}$ is unique (from equation~\ref{eq:P_vector}), element-wise positive and it lies in the eigen-space of the $\hat{\mathbf{P}}$ matrix corresponding to its unity eigenvalue. Hence, the matrix $\hat{\mathbf{P}}$ is irreducible. 
\end{proof}
\begin{theorem}\label{thm:centralcontrol}
Let $P^{\star} \in \mathds{R}^M$ be an element-wise positive probability vector and let $\mathbf{P} \in \mathds{R}^{M \times M}$ be an irreducible stochastic matrix with a stationary probability vector, $P$. Then, $\exists$ a diagonal matrix $D \in \mathds{R}^{M\times M}$ with $D_{ii}\in (0,1)$ such that $\mathbf{P}^\star=D\mathbf{P}-D+I$ is an irreducible stochastic matrix with stationary probability vector $P^\star$. 
\end{theorem} 
\begin{proof}
 We first assume that such a $D$ matrix does exist and prove the claim by finding one. So, let us assume such an $D$ matrix exist. Then, 
\begin{equation}
 \mathbf{P} \mapsto D\mathbf{P}-D+I(=\mathbf{P}^{\star})
\end{equation}
 such that,
\begin{eqnarray}
{P}^{\star}\mathbf{P}^{\star }={P}^{\star}\\
\Rightarrow {P}^{\star}D(\mathbf{P}-I) =0 \\
\Rightarrow {P}^{\star}D ={P} (\text{see Lemma~\ref{lemma1}})
\end{eqnarray}
Since D is a diagonal matrix, we have
\begin{eqnarray}
\begin{aligned}
  \begin{pmatrix}
   {P}^{\star}_1 & \dots & {P}^{\star}_N
  \end{pmatrix}
  \begin{pmatrix}
   d_1 & & &\\
  & &  \ddots \\
& & & d_M
   \end{pmatrix}
  = 
\begin{pmatrix}
   {P}_1 & \cdots & {P}_M
  \end{pmatrix}\\
\Rightarrow \begin{pmatrix}
             d_1 & \cdots & d_M
            \end{pmatrix}
\underbrace{
\begin{pmatrix}
 {P}^{\star}_1 & & & \\
 & &  \ddots\\
& & & {P}^{\star}_M
\end{pmatrix}
}_{\text{say }\mathbb{P}^{\star}}
=
\begin{pmatrix}
 {P}_1 & \dots & {P}_M
\end{pmatrix}\\ \nonumber
\Rightarrow \begin{pmatrix}
             d_1 & \dots & d_M
            \end{pmatrix}
=
\begin{pmatrix}
 {P}_1 & \dots & {P}_M
\end{pmatrix} 
{
\begin{pmatrix}
 {P}^{\star}_1 & & &\\
 & &  \ddots\\
& & & {P}^{\star}_M
\end{pmatrix}
}^{-1} 
\end{aligned}
\end{eqnarray}
where, $\mathbb{P}^{\star -1}$ exists as it is a diagonal matrix with positive entries. We can normalize $d_i$'s to get $d_i\in (0,1)$. Then, the required $D$ matrix is given by ${\rm diag}(d_1,\dots,d_M)$.
\end{proof}
It is easy to see that the matrix $\mathbf{P}^\star$ satisfies the four conditions mentioned in Section~\ref{sec:problem1}. 
\begin{remark}
In Lemma~\ref{lemma1}, we show that under the perturbations described, we retain the irreducibility of the stochastic matrix (thus satisfy condition $1-3$ in Section~\ref{sec:problem1}) and through Theorem~\ref{thm:centralcontrol} we show synthesize a possible perturbation so that the perturbed irreducible stochastic matrix attains the desired distribution for the swarm (thus satisfy condition $4$ in Section~\ref{sec:problem1}). It is noted that the solution is not unique; however, this gives a closed-form solution for controller synthesis. 
\end{remark}
The above calculation could be done in one-shot and thus, presents a one-stage solution to the algorithm proposed for the same problem in~\cite{CR09}. Simulation results demonstrating the correctness of the algorithm are presented in section~\ref{sec:results}. The synthesized controller has asymptotic convergence guarantees and is globally stable (which follows from property of irreducible Markov chains and their stochastic kernels)~\cite{MT12}. However, the controller is implemented in an open-loop, feedforward fashion.
Next we show how to synthesize stable perturbations to the global policy matrix $\textbf{P}^\star$ based on the local information which can help reduce agent movement and can make the system possibly reactive to unmodeled events.

\subsection {Algorithm and Analysis for Distributed Autonomy}\label{sec:distsoln}
In the last section, we presented an analytical solution to the controller synthesis problem described in Section~\ref{sec:problem1}. It was based on the knowledge of global state of the swarm and the synthesized controller has asymptotic convergence and global stability guarantees. However, it might lead to unnecessary movement of agents at steady state  as the controller is implemented in an open-loop fashion and the knowledge of current state is not considered. Furthermore, the swarm cannot react to any unforeseen changes events which might lead to changes in requirement of agents at different states. It is desirable that the robots (agents) should have some degree of autonomy to choose their action based on their current state and local state information of the swarm. This forms the motivation of the problem described in Section~\ref{sec:problem2}. In this section, we present a framework to allow distributed autonomy to the agents so that they can decide to follow the global policy in a probabilistic fashion while retaining global stability. While it is in general difficult to define a unique notion of autonomy for distributed systems, here we will define a very simple notion of the same.
 
 We define the degree of autonomy as the fraction of the times an agent decides to follow the global policy against the policy of staying in the same state. To understand this notion more clearly, imagine the swarm is represented by a two-state Markov chain. Let us further assume that the global policy is given by the following stochastic matrix:
\begin{equation}
B=
\begin{bmatrix}
p_{11} & p_{12} \\
p_{21} & p_{22} \nonumber
\end{bmatrix}
\end{equation}
Then we define the degree of autonomy for a state $i$ by a scalar parameter $a_i \in (0,1)$ such that the final policy used by the agents is represented by the following matrix:
\begin{equation}
\hat{B}=
\begin{bmatrix}
a_1 p_{11}+(1-a_1) & a_1p_{12} \\
a_2 p_{21} & a_2 p_{22}+(1-a_2) \nonumber
\end{bmatrix}
\end{equation}
Thus, the agents decide to use the global policy in a probabilistic fashion, where the probability of following the global policy in a certain state is represented by the degree of autonomy (like $a_i$ in the above example). It is noted that all agents in the same state have the same degree of autonomy.

The idea presented in this section is to synthesize perturbations to the global stochastic policy  (which was synthesized in the last section). These perturbations denote the autonomy (of the corresponding degree) for the individual agents in following the global policy vs staying at the same state. We estimate a state-value function for the Markov chain that represents the swarm, based on the requirement of agents at the individual states. 

The local policy represents the probability with which an agent decides to either follow the global policy or stay in its current state. The goal is to achieve a pre-defined percentage of activity at steady-state (as described above in the notion of autonomy). To design such a controller, we use the state value function defined by language measure theory (see~\cite{J18} for a brief introduction). To formulate the problem of distributed autonomy using language measure theory, we augment the states of the Markov chain (that represents the swarm) with characteristic weights 
The augmented Markov chain then defines a probabilistic finite state automata (PFSA) which we use to define the state value function for our problem using language measure of the PFSA (see~\cite{JLWR15,chattopadhyay2015path, jha2016path}). The characteristic weights of the states are defined as follows:
\begin{equation}\label{eqn:chi}
\chi_i^{[k]}= e_i^{[k]} 
\end{equation} 
The vector $\mathbf{\chi}^{[k]}=[\chi_1^{[k]},\chi_2^{[k]},\dots,\chi_M^{[k]}]^T$ thus contains the information about the deficit or excess of agents at the individual tasks (or the states of the swarm). A positive value of $\chi$ would suggest deficit of agent in one state and a negative value of $\chi$ suggests excess of agents. The states with positive values of $\chi$ represent the good states to move to from a neighboring state. A greedy policy would thus try to move to the states with the maximum positive characteristic weights. As mentioned earlier, the goal is to achieve a certain fraction, $\lambda$, of original activity level at steady state. During the transient phase, some states with positive characteristic weights observe higher (than steady-state) activity rates than states with negative characteristic weights. It is presented more formally next.

 Instead of using a greedy policy, we compute the expected characteristic weights for the states which represent the overall expected rewards (where the expectation is computed over the states using the stochastic transition matrix of the agents) that the agents can accumulate while trying to achieve the target state distribution. 
 The expected value of the characteristic weights for the Markov chain (that represents the swarm) is calculated using the language measure theory. The expectation, parameterized by a parameter $\theta \in (0,1)$, of the characteristic weights of an irreducible Markov chain with stochastic matrix $\mathbf{P}$ is calculated by the following recursive equation.
\begin{equation}\label{eqn:measure}
\nu_i\leftarrow (1-\theta)\sum_{j \in {\mathrm {Nb}}(i)}\mathbf{P}_{ij}\nu_j +\theta \chi_i
\end{equation} 
where, ${\mathrm {Nb}}(i)$ is the neighborhood of the agent $i$ (Note that this requires computation over the nearest neighbors only and thus, doesnt require global information). A value of $\theta$ close to $0$ is selected for computing the expected sum. Interested readers are referred to earlier publications~\cite{CR07} for further discussion on choice of the parameter $\theta$. The parameter $\theta$ governs the horizon length for computation of expectation-- for values closer to $1$, we approach the greedy policy. We use $\theta=0.02$ for the computation of the expectation. In the next section, we will show that $\theta$ closer to $1$ makes system convergence slow but more stable to higher rates of feedback. The measure $\nu_i$ is the discounted expected value of $\chi$ for agents starting in state $i$. Clearly, $\mathbf{\nu}=\mathbf{0}$ if $\mathbf{\chi}=\mathbf{0}$. Convergence of the expected weights follow from the fact that $\mathbf{\chi}$ is a constant vector and $\mathbf P$ is a row stochastic matrix (see previous work in~\cite{JLWR15,  chattopadhyay2011goddes} for proofs on convergence). All the states (i.e., tasks) synchronously calculate their own $\nu_i$ and then broadcast it to their neighbors. This is repeated recursively till the expectations converge. Based on the measure defined in \eqref{eqn:measure}, we define a quantity, $\mathbf{\mu}=\mathbf{\nu}-\mathbf{\chi}$. 

\begin{algorithm}[t] \small
\SetKwInput{KwIn}{Input}
\SetKwInput{KwOut}{Output}
	\KwIn{The initial Markov kernel $\mathbf{P}$, the error vector $\mathbf{\chi}^{[k]}$ and the proportionality constant $\beta$.}
	\KwOut{The decentralized policy kernel $\mathbf{P}^\star$}
	\For{$k \in \mathbb{N}$}{
	\For{$i \in \{1,.\dots,M\}$ $\text{till convergence}$}{
	\For{$j \in \tt {Nb}(i)$}{
	\Comment{$\tt{Nb}(i)$ represents the neighborhood of state $i$}\\
	$\nu_i^{[k]}\leftarrow \sum\limits_{j \in \tt {Nb}(i)}(1-\theta)\mathbf P_{ij}\nu_j+\theta \chi_i^{[k]}$\;}}
	$\mu^{[k]}=\nu^{[k]}-\chi^{[k]}$\;
	\For{$i \in \{1,.\dots,M\}$}{
	$f_\lambda(\mu_i^{[k]})=\frac{1}{1+(\frac{1}{\lambda}-1){\mathrm {exp}}(-\beta^{[k]} \mu_i^{[k]})}$\;}
	\Return{$b_i^{[k]}=\{f_\lambda(\mu_i^{[k]})\}_{i=1,\dots,M}$\;}
	\Return{$\mathbf P^\star=B\mathbf P -B +I$, where $B=\mathrm{diag (\mathbf b^{[k]})}$ }}
	\caption{${\tt Estimating Distributed Policy}$ }
	\label{alg:DistPolicy}
\end{algorithm}	

The quantity $\mu_i$ represents the difference between the characteristic weight for state $i$ and the expected value of the $\chi$ for agents starting in state $i$. As such, a positive value of $\mu_i$ would mean that the states to which the agents can go from that state have higher expected reward than their current state and hence, such states are expected to have higher activity. An activation function is defined according to the following sigmoid function
\begin{equation}\label{eqn:activitylevel}
 f_\lambda(\mu_i^{[k]})=\frac{1}{1+(\frac{1}{\lambda}-1){\mathrm {exp}}(-\beta^{[k]} \mu_i^{[k]})}
\end{equation}
where, $\lambda \in (0,1)$ and $\beta^{[k]} \in \real_+$ are parameters. $\lambda$ denotes the steady state activity level for the agents. $\beta^{[k]}$ is a scaling factor for $\mu_i$. Clearly, $f_\lambda(\mu_i^{[k]})\in (0,1)$ for all $\mu_i^{[k]} \in \real$. The sigmoid functions are used to design smooth (i.e., the feedback rate is differentiable) feedback controllers. Figure~\ref{fig:activity} shows the variation in the change in activity of agents with scaling factor $\beta$ (superscript $k$ dropped for ease of notation). It can be seen in Figure~\ref{fig:activity} that higher values of $\beta$ lead to an increase in the rate of change of activity and thus, it can lead to instability (or oscillations) in the system when increased beyond an unknown threshold. In the next section, we will show that the system can get locked in limit cycles of different amplitudes based on the magnitude of the feedback rate used for controller design. The stable feedback rates depend on the size of the system, the parameter $\theta$ for computing the expectation and the parameter $\lambda$.

\begin{figure}[h]
\centering
 \includegraphics[scale=0.35]{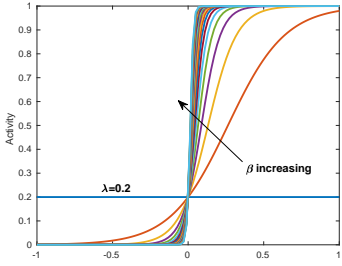}
\caption{This plot shows the effect of the parameter $\beta^{[k]}$ on the rate of change of activity.}
\label{fig:activity}
\end{figure}\vspace{+0pt}

 Thus, the feedback matrix for at epoch $k$ is given by,
\begin{eqnarray}\label{eqn:feedbackCont}
\tilde{\textbf{P}}^{[k]}_{ij} = b_i^{[k]} \textbf{P}^\star_{ij}, i\neq j \nonumber \\
\tilde{\textbf{P}}^{[k]}_{ij} = b_i^{[k]} \textbf{P}^\star_{ij}+(1-b_i^{[k]}), i=j 
\end{eqnarray}
where, $b_i^{[k]}=f_\lambda(\mu_i^{[k]})$ and $\mu_i^{[k]}$ was defined earlier. Thus the idea is that based on the state of a task and the neighboring tasks, an agent can probabilistically decide whether to follow the global policy or stay in the same task. Note that the perturbation in~\eqref{eqn:feedbackCont} satisfies conditions $1-3$ in Section~\ref{sec:problem2}. Clearly, 
$\lim\limits_{k\rightarrow \infty} f_\lambda(\mu_i^{[k]})=\lambda \text{ if } \lim\limits_{k\rightarrow \infty} \mu_i^{[k]}=0$. At steady state, $\nu=\chi=0$ and thus, $f_\lambda(\mu_i^{[k]})=0$ for all $i \in \{1,\dots,M\}$. Thus, at steady state (or the desired state), this should lead to $\lambda$ fraction of agents staying in their task by deciding against the global policy to switch. At every iteration, the agent activity level $b_i^k$ decides the probability with which the agents in task $i$ decide against the global policy and stay in the same task; this results in overall reduced activity. The algorithm for computing the distributed policy is also presented as a pseudo-code in Algorithm~\ref{alg:DistPolicy}.

At steady state, $\chi=0$ and $\nu=0$ and thus $\tilde{\mathbf{P}}=\lambda \mathbf{P}^\star-\lambda I + I$. It is easy to see that the stationary distribution of the matrix $\mathbf P^\star$ is also the stationary distribution of the steady-state matrix $\tilde{\mathbf{P}}$. More concretely, if $\mathbf{p}^d$ is the stationary distribution of $\mathbf{P}^\star$, then $\mathbf{p}^d \tilde{\mathbf{P}}= \mathbf{p}^d(\lambda \mathbf{P}^\star-\lambda I + I)= \lambda \mathbf{p}^d+(1-\lambda)\mathbf{p}^d=\mathbf{p}^d$. This satisfies condition $5$ specified in Section~\ref{sec:problem2}.\\

\textbf{Design of Stable Feedback:} Next we consider the convergence of the feedback controller represented by condition $4$ in Section~\ref{sec:problem2}. We first present a candidate Lyapunov function to present the sufficient conditions for system stability. A possible candidate is described next.
\begin{equation}\label{eqn:lyapunov}
V^{[k]}(\mathbf{p})=\sum_{i=1}^M(p_i^d-p_i^{[k]})^2
\end{equation} 
From equation~(\ref{eqn:lyapunov}), we have $V^{[k]}(\mathbf{p}^d)=0$ and $V^{[k]}(\mathbf{p})\geq 0$ for all $\mathbf{p}$. Then, for system stability it is required that $\Delta V^{[k]}(\mathbf{p})\leq 0$ for all $k\in \mathds N$. From equation~(\ref{eqn:lyapunov}), we get the following expression for $\Delta V^{[k]}(\mathbf{p})$.
\begin{equation}\label{eqn:lyapunovstab}
\Delta V^{[k]}(\mathbf{p}) =V^{[k]}(\mathbf{p})-V^{[k-1]}(\mathbf{p})
\end{equation}
Further expansion of equation~\eqref{eqn:lyapunov} leads to the following form 
\begin{equation}\label{eqn:LyapStab1}
\Delta V^{[k]}(\mathbf{p}) = \sum\limits_{i=1}^M(2p_i^d-p_i^{[k]}-p_i^{[k-1]})(p_i^{[k-1]}-p_i^{[k]})
\end{equation}
In equation~\eqref{eqn:LyapStab1}, we introduce the following notations $e_i^{[k]}= p_i^d-p_i^{[k]} $ and $\Delta e_i^{[k]}= e_i^{[k]}-e_i^{[k-1]}$. Then, using the Lyapunov condition for stability, the swarm is stable if the following holds true.
\begin{equation}\label{eqn:LyapStab2}
\sum\limits_{i=1}^M(e_i^{[k]}+e_i^{[k-1]})\Delta e_i^{[k]} \leq 0
\end{equation}
Based on the definition of $e_i^{[k]}$, it is not difficult to see that a sufficient condition for equation~\eqref{eqn:LyapStab2} to hold true is that 
\begin{equation}\label{eqn:Lyapstab3}
e_i^{[k]}\Delta e_i^{[k]}\leq 0
\end{equation}
 for all $i \in \{1,\dots, M\}$. Let us denote the sequence $e_i^{[k]}\Delta e_i^{[k]}$ by $\zeta_i^{[k]}$. Thus a sufficient condition for stability of the distributed system is that there exists a $N\in \natural$ such that the sequence $\{\zeta_i^{[k]}\}_{k\geq N}$ is monotonic and as $\zeta_i^{[k]}$ is upper bounded by $0$, we have that $\lim\limits_{k\rightarrow \infty}\zeta_i^{[k]}=0$ for all $i \in \{1,\dots, M\}$. Then, to ensure stability we design the perturbations to the global policies so that the above condition holds. It is noted that this only defines a sufficient condition for stability of the system.
Based on the condition expressed in~\eqref{eqn:Lyapstab3}, we need to ensure that for all $\beta \in \real$ we can finally ensure that there exists a $N\in \natural$ such that $\{\zeta_i^{[k]}\}_{k\geq N}$ is monotonic. Loosely speaking it means that the feedback rate around the equilibrium point be sufficiently reduced. The analysis of sigmoid functions for stability of neural networks has been well studied in literature. In some analysis present in literature, it has been shown that if the feedback rate should be smaller than inverse of the second largest eigenvalue of the stochastic matrix~\cite{JG96}. With this motivation, we can design a controller whose feedback rate decreases as we move closer to the regulated state.

In the controller design, we don't need to find the actual value of maximum allowable $\beta^{[k]}$ to ensure stability; the existence is enough to guarantee global asymptotic stability. The idea is to reduce the feedback constant rate $\beta$ as the system moves towards the desired state-- thus reducing the activity of the states as the system moves towards the regulated state. A very simple feedback rate dependent on the iteration number is $\beta^{[k]}=\gamma/k$, where $\gamma \in \real_+$ is a constant. Similarly, an exponentially decreasing $\beta^{[k]}$ could also work well (e.g., $\beta^{[k]}=\gamma*\exp^{(-k/N)}$, where $N$ is some large constant). As a result of this feedback rate design, the feedback rate eventually reduces to a level for monotonic convergence of the system irrespective of the size of the system. 

\begin{figure}[h]
\centering
 \includegraphics[scale=0.25]{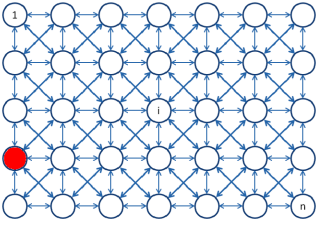}
\caption{The tasks modeled as a graph with its connectivity topology. Initially all the agents are in the state marked in red; the desired distribution is to uniformly distribute them over all tasks.}
\label{fig2:Graph}
\end{figure}\vspace{+0pt}
\section{Numerical Results and Discussion}\label{sec:results}
In this section, we present some numerical results based on the proposed algorithms described in sections~\ref{sec:closedsoln} and~\ref{sec:distsoln}. Through the numerical experiments, we try to answer the following questions about the controller design that was presented in the earlier sections:
\begin{enumerate}
    \item Can we show the correctness of the proposed central controller design? Does it achieve the desired global behavior?
    \item Can we understand and ensure stability of the proposed distributed controller?
\end{enumerate}
To answer the above mentioned questions, we design numerical experiments in a swarm which is described next. We consider a task allocation problem which consists of the $35$ tasks being operated in parallel. The connectivity of tasks is represented as a graph which is shown in Figure~\ref{fig2:Graph}. Each node of the graph (or a task) is connected to $8$ neighboring nodes (or tasks) except for the ones on the edges which are connected to either $5$ or $3$ nodes (see Figure~\ref{fig2:Graph}). 

\begin{figure}[h]
\centering
\subfloat[Convergence of the agent distribution to the desired distribution]{
\includegraphics[scale=0.25]{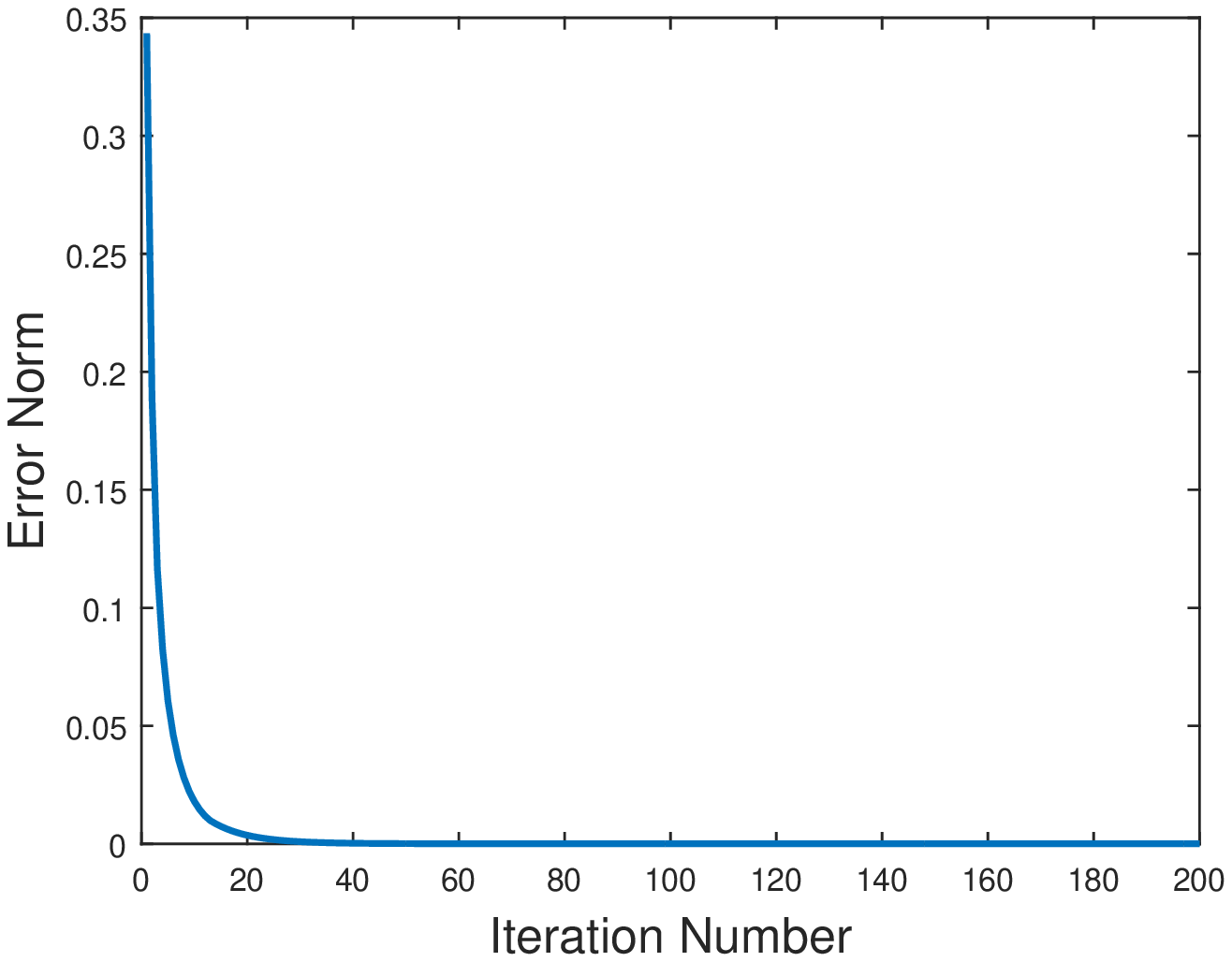}\vspace{-0pt}
\label{fig2a:figCC1}
}\vspace{+0pt}
\centering
\subfloat[Activity level at steady state]{
\includegraphics[scale=0.25]{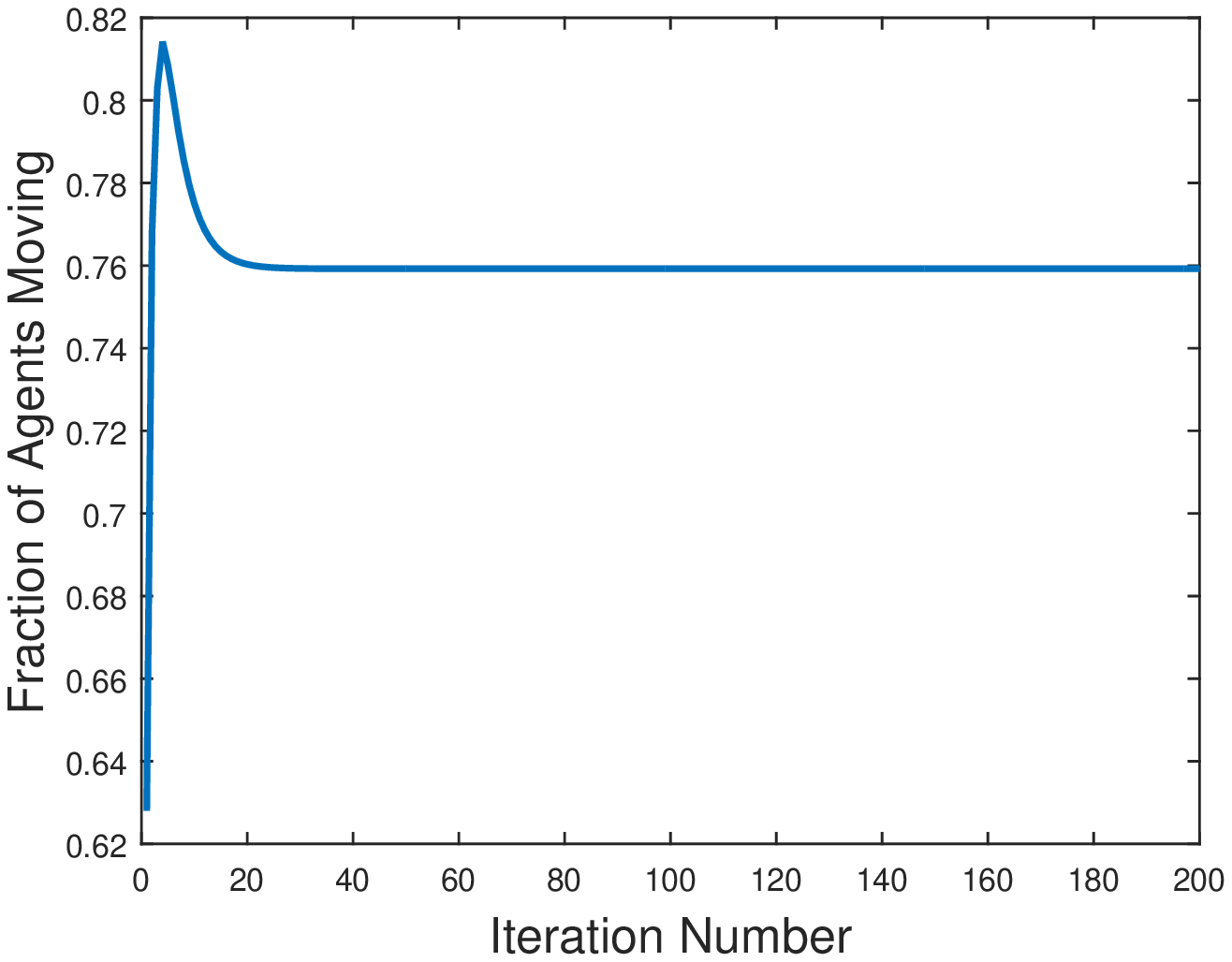}\vspace{-0pt}
\label{fig2b:figCC2}
}\vspace{+0pt}
\caption{Convergence of system states and agent activity obtained by the central control policy}\vspace{-0pt}
\label{fig2:figCC}
\end{figure}\vspace{+0pt}
\begin{figure}[h]
\centering
\subfloat[A high proportional gain in the local policy synthesis results in steady-state error with oscillatory behavior]{
\includegraphics[scale=0.25]{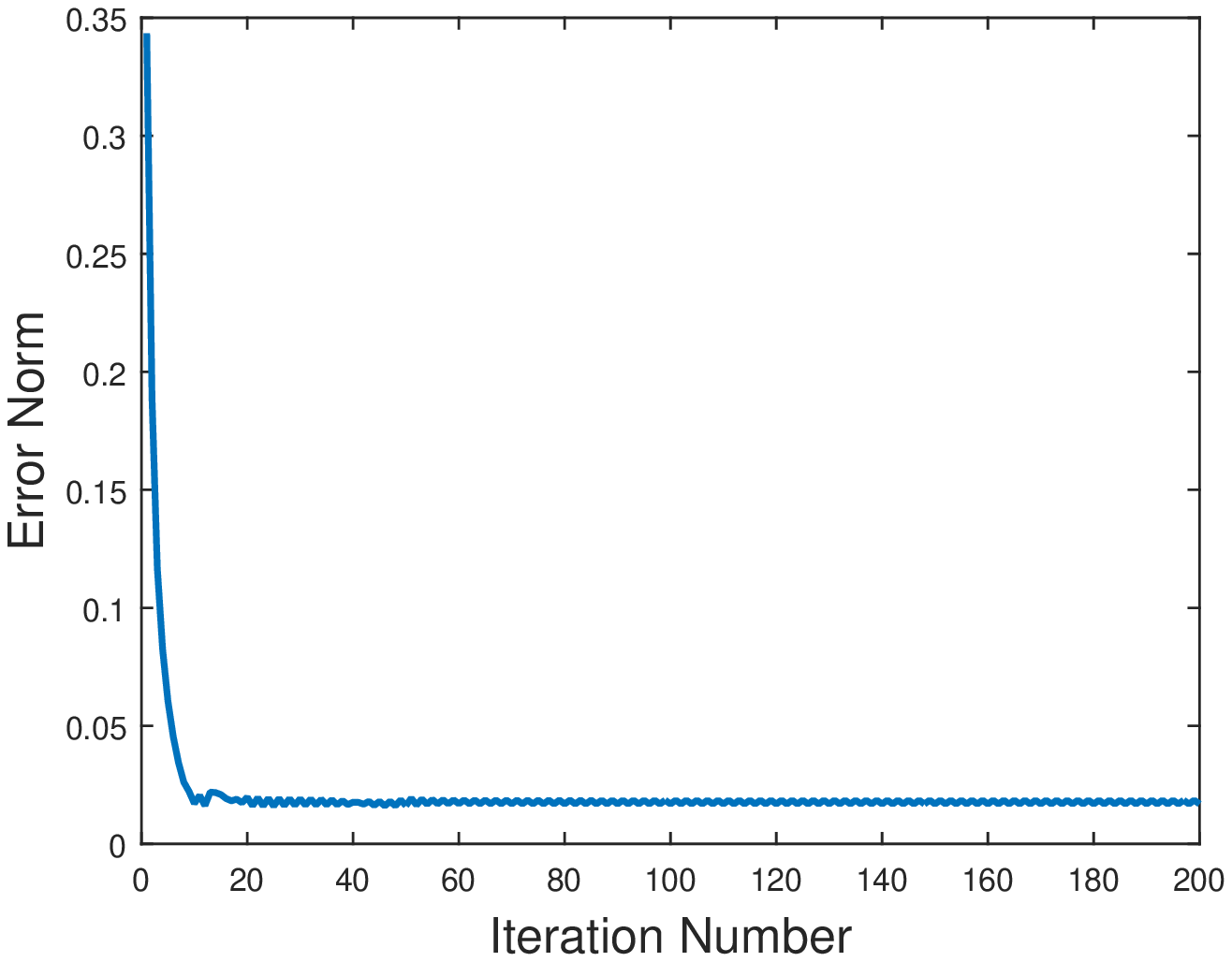}\vspace{-0pt}
\label{fig2a:figDCUS1}
}\vspace{+0pt}
\centering
\subfloat[Activity level at steady state]{
\includegraphics[scale=0.25]{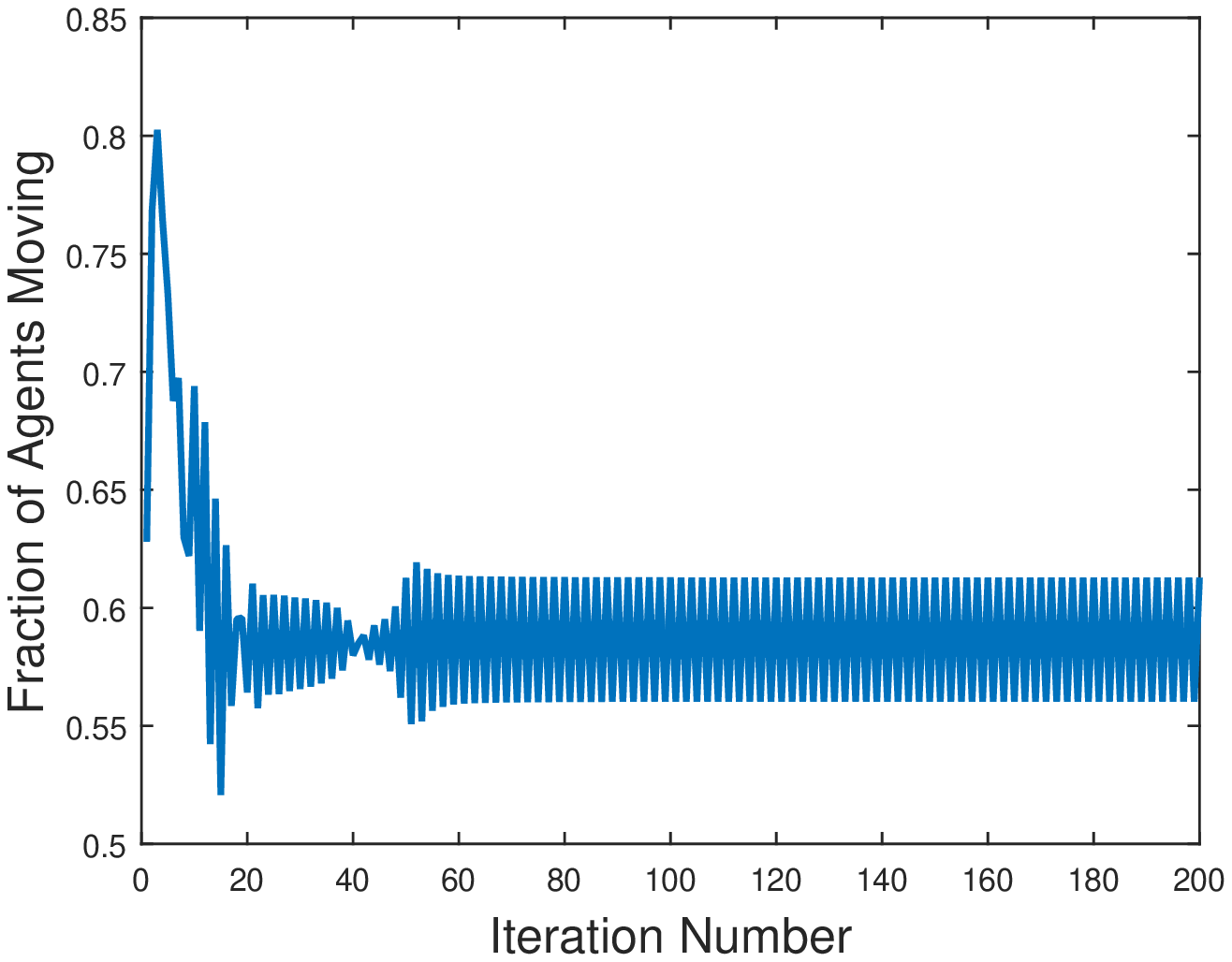}\vspace{-0pt}
\label{fig2b:figDCUS2}
}\vspace{+0pt}
\caption{Results corresponding to a fixed high proportional gain leading to oscillatory behavior}\vspace{-0pt}
\label{fig2: figDCUS}
\end{figure}\vspace{+0pt}
\begin{figure}[h]
\centering
\subfloat[Convergence of the agent distribution to the desired distribution over the parallel tasks]{
\includegraphics[scale=0.25]{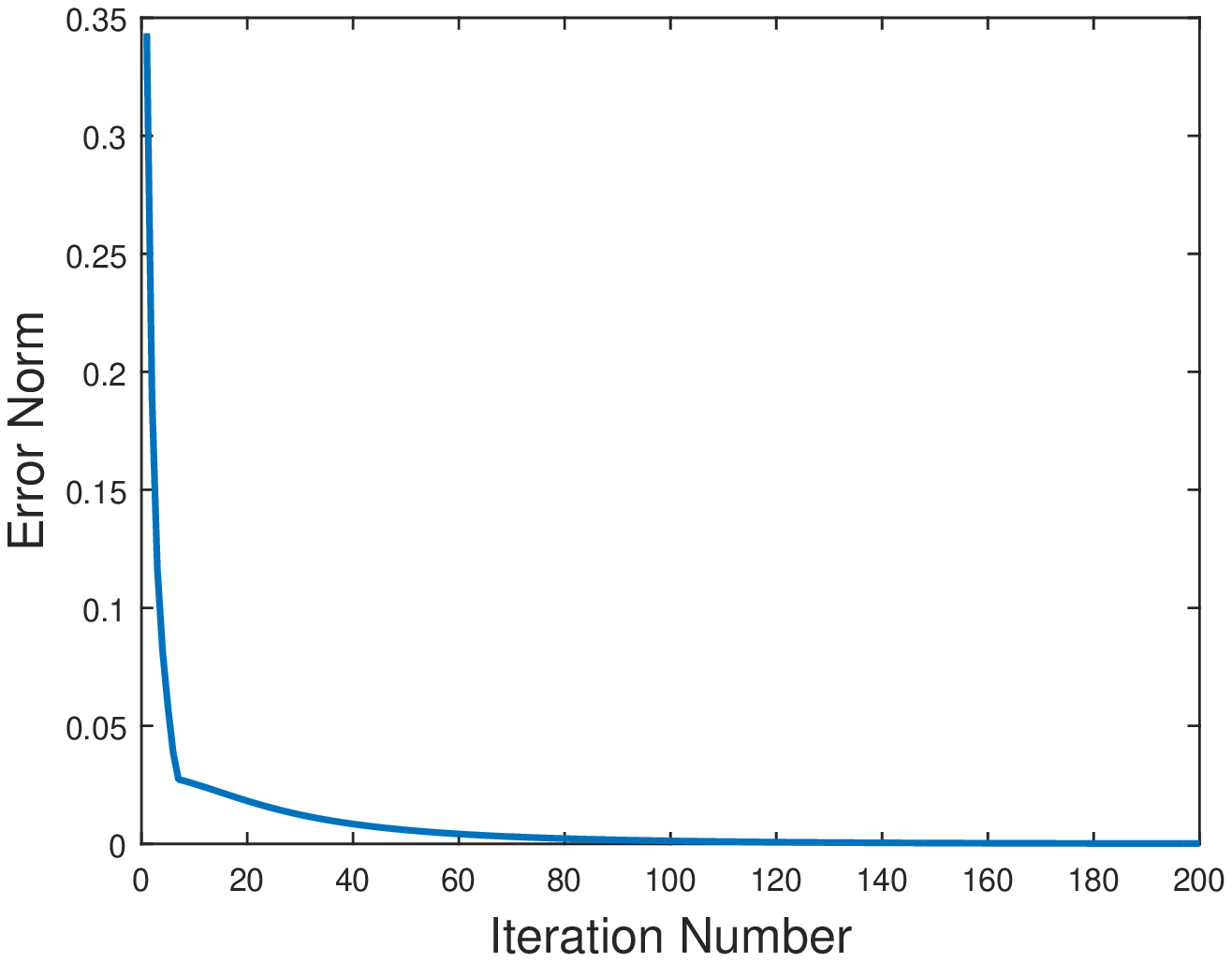}\vspace{-0pt}
\label{fig2a:figDCS1}
}\vspace{+0pt}
\centering
\subfloat[Activity level at steady state]{
\includegraphics[scale=0.25]{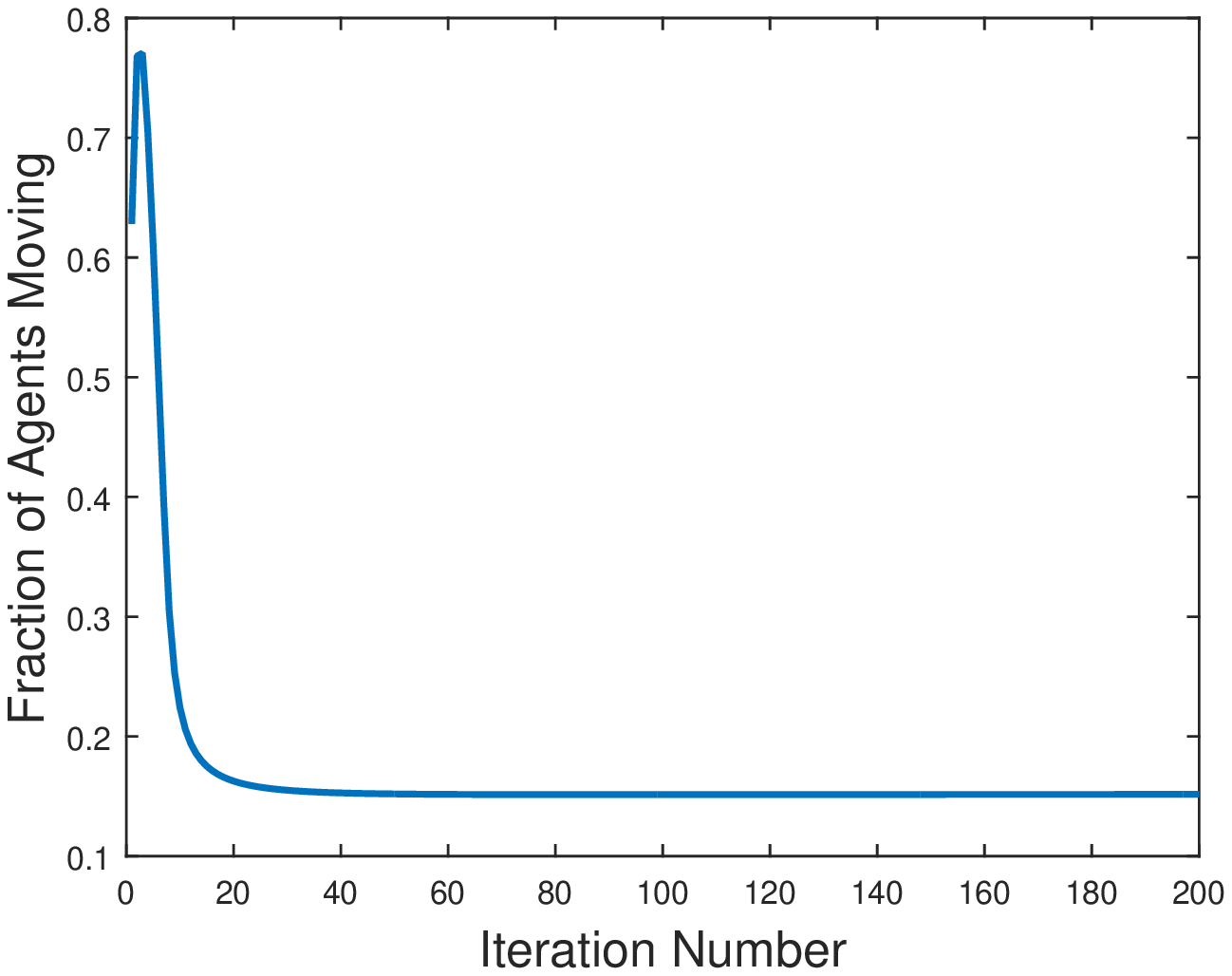}\vspace{-0pt}
\label{fig2b:figDCS2}
}\vspace{+0pt}
\caption{Convergence and agent activity level obtained by the distributed feedback controller with $\beta^{[k]}=\gamma/k$, where $\gamma=600$.}\vspace{-0pt}
\label{fig2: figDCS}
\end{figure}\vspace{+0pt}


\begin{figure}[h]
\centering
 \includegraphics[scale=0.25]{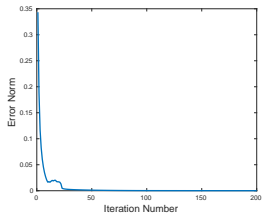}
\caption{Convergence of agent movement with exponentially decaying $\beta^{[k]}=\gamma*\exp^{(-k/N)}$, where $\gamma=2000$, $N=100$.}
\label{fig:fig4}
\end{figure}\vspace{+0pt}

\begin{figure}[h]
\centering
 \includegraphics[scale=0.25]{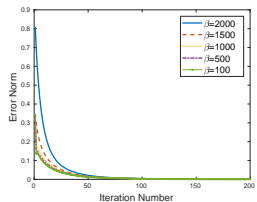}
\caption{Convergence of the distributed controller for various values of $\beta$ for $\theta=0.98$. Clearly, higher values of $\theta$ are slower but more stable to feedback rates. }
\label{fig:fig5}
\end{figure}\vspace{+0pt}

Initially all the agents are located at the same task which is marked in red in the graph. The agents need to be distributed uniformly over all the tasks at steady state. In Figure~\ref{fig2:figCC} we show the results obtained for the central controller. The system converges to the desired distribution monotonically (see Figure~\ref{fig2a:figCC1}) ; however, there is a lot of activity at steady state which is undesirable (see Figure~\ref{fig2b:figCC2}).  

Next we consider the distributed control which is calculated as perturbation to the central control policy discussed earlier. For the distributed control, the parameter $\theta$ is selected to be $0.02$ and the desired activity level is taken to be $\lambda=0.2$ i.e., only $20\%$ of the agents should move at steady state when compared to the movement shown in Figure~\ref{fig2b:figCC2}. We first show a controller which is unstable due to very high gain which leads to unstable (oscillatory) behavior in the swarm around the desired state. To see this, we select a high proportional gain corresponding to $\beta=600$ which is kept constant through all iterations i.e., $\beta^{[k]}=600$. The results (see Figure~\ref{fig2: figDCUS}) show the unstable or oscillatory behavior of the system; the agents keep moving between the tasks without ever reaching the steady state. As it is shown in Figure~\ref{fig2a:figDCUS1}, this results in steady state error. As seen in Figure~\ref{fig2b:figDCUS2}, the agent activity becomes oscillatory and is not able to achieve $\lambda$ fraction of activity at steady state. The reason for the oscillatory behavior is the high proportional feedback gain which is being used for estimating feedback. With a high-feedback gain the system is never able to reach the desired set-point and keeps over-shooting. This also motivates the Lyapunov-based design where, a sufficient condition for stability is reached by the proportional feedback gain given by $\beta^{[k]}=\gamma/k$ for $\gamma \in \real$.

Next we show the use of proportional feedback given by $\beta^{[k]}=\gamma/k$ where $\gamma=600$ which guarantees asymptotic convergence. In Figure~\ref{fig2: figDCS}, we show the results of the stable distributed proportional controller. As seen in Figure~\ref{fig2a:figDCS1}, the error norm asymptotically converges to zero. Figure~\ref{fig2b:figDCS2} shows the agent activities at steady state reaches $\lambda$ fraction of the activity achieved by the central control policy. It is noted that the system convergence rate is slower than the global policy; however we reach with the desired activity with lesser activity and a desired reduced activity can be achieved at steady state. A similar result with an exponentially decaying feedback rate is also shown as an example in Figure~\ref{fig:fig4}. Compared to the central control, the distributed controller is slower (as seen by the error convergence rates in Figures~\ref{fig2:figCC} and~\ref{fig2a:figDCUS1}). However, the distributed controller is able to achieve the target state with reduced movement of agents and maintains reduced activity at steady state. In general, using an exponentially decaying feedback allowed use of more aggressive feedback rates in the beginning and thus leads to faster convergence.

Next we present some results related to the dependence on the system convergence rate and stability on the parameter $\theta$. As discussed earlier, the parameter $\theta$ governs the horizon of the expected sum while computing the state measure. A value of $\theta$ closer to $1$ leads to a greedy policy by the agents (as the expected sum depends largely on the immediate rewards). Figure~\ref{fig:fig5} shows convergence of the distributed controller for $\theta=0.98$ and different constant values of $\beta$. As can be seen, with increasing values of $\beta$, the system convergence becomes slower for values of $\theta$ closer to $1$. This becomes clear from the expression of $\nu_i$ in Equation~\eqref{eqn:measure} where the coefficient of $\chi_i$ is $\theta$. As a result, the term $\nu_i$ is dominated by $\chi_i$ and thus $\mu_i$ is very close to $0$. This allows use of higher rates in Equation~\eqref{eqn:activitylevel}. This is however, opposite of what was observed with $\theta$ close to $0$. In general, using an exponentially decaying feedback rate with $\theta$ close to $0$ lead to best performance as well as stability of the distributed system (seen in Figure~\ref{fig:fig4}). 

\section{Conclusions}\label{sec:conclusions}
In this paper, we modeled a homogeneous swarm as a collection of irreducible Markov chains. The problem of synthesizing controller for the swarm is then equivalent to estimating the Markov kernel of the Markov chain whose stationary distribution is the desired swarm state. We presented analysis for an analytical solution to the centralized swarm control problem which greatly alleviates controller synthesis computational complexity for swarms with large number of states. Next we proposed a local-information-based perturbation to the centralized controller which can achieve reduced activity with guarantees of global asymptotic stability. 


 
\bibliographystyle{ieeetr}
\bibliography{references}
\end{document}